\documentclass[namedreferences,hyperref,optionalrh]{springer}
\usepackage{graphicx}  
\usepackage{color}      
\usepackage{float}

\newcommand{\std}[1]{^{\scriptstyle{\pm#1}}}

\usepackage{algorithm, algpseudocode}
\usepackage{mathtools}
\usepackage{url}
\usepackage{soul}
\usepackage{bbm}
\usepackage{amsmath}
\usepackage{amssymb}
\usepackage{amsthm}
\usepackage{adjustbox}
\usepackage{booktabs}
\newtheorem{theorem}{Theorem}

\newtheorem{proposition}[theorem]{Proposition}

\newcounter{algsubstate}
\renewcommand{\thealgsubstate}{\alph{algsubstate}}

\def \score {\textrm{conformity score}}

\def \Atrain {A^{\textrm{train}}}
\def \Aval {A^{\textrm{val}}}

\def \Atest {A^{\textrm{test}}}

\def \Wtrain {W^{\textrm{train}}}
\def \Wval {W^{\textrm{val}}}

\def \Wtest {W^{\textrm{test}}}
\def \Wcalib {W^{\textrm{calib}}}

\def \Etrain {E^{\textrm{train}}}
\def \Eval {E^{\textrm{val}}}
\def \Ect {E^{\textrm{ct}}}
\def \Etest {E^{\textrm{test}}}
\def \Ecalib {E^{\textrm{calib}}}

\def \loss {\mathcal{L}}

\usepackage{comment}
\newif\ifshow
\showtrue
\ifshow
  
\else
  \excludecomment{allcomments}
\fi

\begin{document}

\begin{frontmatter}
\title{Conformal Load Prediction with Transductive Graph Autoencoders}

\author[addressref={aff1},corref,email={ruiluo@cityu.edu.hk}]{\inits{R.L.}\fnm{Rui}~\snm{Luo}}
\author[addressref=aff2,email={nicolo.colombo@rhul.ac.uk}]{\inits{N.C.}\fnm{Nicolo}~\snm{Colombo}}
\address[id=aff1]{City University of Hong Kong, Kowloon Tong, Hong Kong SAR}
\address[id=aff2]{Royal Holloway, University of London, Egham, Surrey, UK}

\runningauthor{Luo and Colombo}
\runningtitle{Conformal Load Prediction with Transductive Graph Autoencoders}

\begin{abstract}
Predicting edge weights on graphs has various applications, from transportation systems to social networks. 
This paper describes a Graph Neural Network (GNN) approach for edge weight prediction with guaranteed coverage. 
We leverage conformal prediction to calibrate the GNN outputs and produce valid prediction intervals.
We handle data heteroscedasticity through error reweighting and Conformalized Quantile Regression (CQR). 
We compare the performance of our method against baseline techniques on real-world transportation datasets. 
Our approach has better coverage and efficiency 
 than all baselines and showcases robustness and adaptability.
\end{abstract}
\end{frontmatter}

\section{Introduction}

Graph machine learning has seen a surge in interest with the advent of complex networked systems in diverse domains.
Applications include social and transportation networks and various kinds of biological systems.
In most cases, the interaction between nodes is typically represented by edges with associated weights. 
The edge weights can embody varying characteristics, from the strength of interaction between two individuals in a social network to the traffic capacity of a route in a transportation system. 
The prediction of the edge weights is vital to understanding and modelling graph data.

Graph Neural Networks (GNNs) have been successfully used on node classification and link prediction tasks.
In this work, we consider their application to edge weight prediction. 
Edge weight prediction has found use in diverse domains such as message volume prediction in online social networks \citep{hou2017deep}, forecasting airport transportation networks \citep{mueller2023link}, and assessing trust in Bitcoin networks \citep{kumar2016edge}. These examples highlight the wide-ranging applicability and importance of edge weight prediction and load forecasting techniques in different network-based systems.

Applying GNNs to edge weight prediction is often unreliable.
Producing prediction intervals with finite-sample guarantees can be useful in many scenarios, e.g. when the GNNs forecast influences a decision-making process.  
In a read transportation network, the prediction intervals may be interpreted as the upper and lower bounds of the predicted traffic flow.
How to integrate this information to support downstream optimization algorithms goes beyond the scope of this work.

We present a novel approach for edge weight prediction with guaranteed coverage.
Focusing on the transductive setting, we define a series of GNN approaches to predict the edge weights of a given graph.
We show how to calibrate the GNN predictions with different conformal inference methods.
The final output of our algorithms is a set of marginally valid prediction intervals for the unknown weights of the graph edges. 
We handle heteroscedastic node features with a new error-reweighted extension of Conformalized Quantile Regression. 

We validate our algorithms empirically using two real-world transportation datasets. 
The proposed approach outperforms all the baseline methods on coverage and efficiency.

The rest of the paper is organized as follows. Section \ref{sec: problem} provides background on GNNs and edge weight prediction. 
Section \ref{sec: conformal} outlines our conformal load forecast methods. 
Section \ref{sec: empirical} presents the experimental results.
Section \ref{sec: conclusion} contains a summary of our contribution and a discussion of potential future directions.

\section{Transductive Edge Weight Prediction Using GNNs}\label{sec: problem}
Let $G=(V, E)$ be a graph with node set $V$ and edge set $E \subseteq V \times V$.
Assume the graph has $n$ nodes with $f$ node features.
Let $X \in \mathbb{R}^{n\times f}$ be the node feature matrix, and $X_i \in \mathbb{R}^{f}$ the feature vector of the $i$th node. 
The binary adjacency matrix of $G$, 
\begin{equation}
A \in \{0, 1\}^{n\times n}, \quad 
A_{ij} =
\begin{cases}
1, & \textrm{if } (i, j) \in E; \\
0, & \textrm{otherwise}.
\end{cases}
\end{equation}
encodes the binary (unweighted) structure of the graph.

We define the weight matrix as $W \in \mathbb{R}_{\geq 0}^{n\times n}$, where $W_{ij}$ denotes the weight of the edge connecting node $i$ to node $j$. In a road system, we interpret $W_{ij}$ as the volume of traffic transitioning from junction $i$ to junction $j$.

We split the edge set into three subsets $E = \Etrain \cup \Eval \cup \Etest$.
We assume we know the weights of the edges in $\Etrain$ and $\Eval$.
The goal is to estimate the unknown weights of the edges in $\Etest$.
We also assume we know the entire graph structure, $A$. 

To mask the validation and test sets, we define \begin{equation}
\Atrain \in \{0, 1\}^{n\times n}, \quad 
\Atrain_{ij} =
\begin{cases}
1, & \textrm{if } (i, j) \in \Etrain; \\
0, & \textrm{otherwise}.
\end{cases}
\end{equation}
Similarly, we let $\Aval$ and $\Atest$ be defined as $\Atrain$ with $\Etrain$ replaced by $\Eval$ and $\Etest$.

Even if $(i, j) \notin \Etrain$, it is possible to assign a positive number $\delta>0$, such as the minimum or average of the existing edge weights, to $\Wtrain_{ij}$ to represent prior knowledge or assumptions about the unknown edge weight. 
This processing is tailored to transportation applications, characterized by a stable graph structure where altering roads is challenging. The focus lies on predicting edge weights. For the edges in the calibration, test, and even validation sets during the training phase, we assign a positive edge weight rather than zero. This approach ensures the model recognizes these connections or the graph structure. An ablation study in Section \ref{sec: empirical} compares two methods of weight assignment: one using the average weight of training edges, i.e., $\delta = \frac{\sum_{(i,j)\in \Etrain} W_{ij}}{|\Etrain|}$; and the other bootstrapping from these weights $\{W_{ij}\}_{(i,j)\in \Etrain}$. We demonstrate that both methods surpass the baseline which does not account for graph structure and sets edge weights to zero, i.e., $\delta=0$.

The resulting weighted adjacency matrix is
\begin{equation}\label{eq: weighted adj train}
\Wtrain =
\begin{cases}
W_{ij}, & \textrm{if } (i, j) \in \Etrain; \\
\delta, & \textrm{if } (i, j) \in \Eval \cup \Etest; \\
0, & \textrm{otherwise},
\end{cases}
\end{equation}

In the transductive setup, the structure of the entire graph, $A$, is known during training, validation, and testing.
To calibrate the prediction, we extract a subset from $\Etest$ as a calibration edge set. This guarantees calibration and test samples are exchangeable, provided 
\begin{itemize}
    \item we do not use $W_{ij}$, $(i,j) \in \Etest$ to make a prediction and 
    \item we split the edge set uniformly at random, and $\Atrain$, $\Aval$, $\Atest$ are exchangeable. %
\end{itemize}

\begin{figure}
\centerline{\includegraphics[width=0.8\textwidth,clip=]{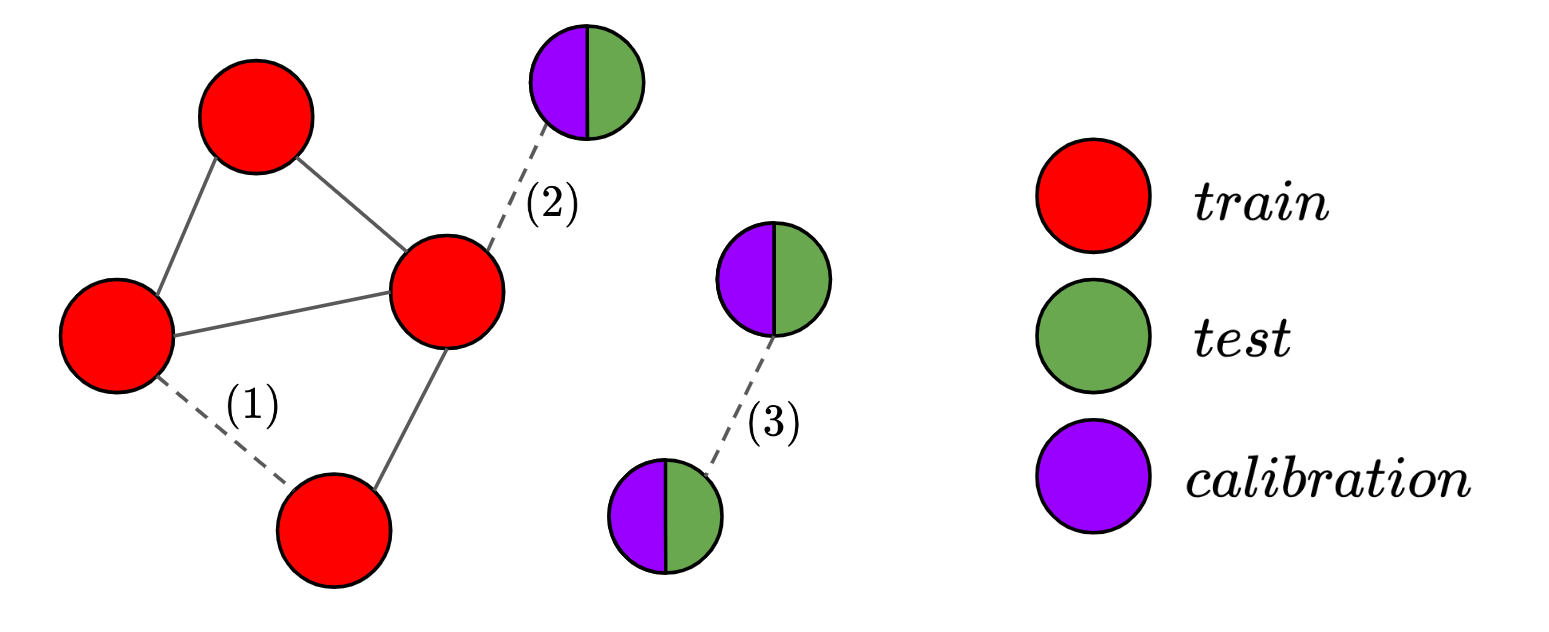}}
\small
\caption{Training settings for edge weight prediction in a conventional data split. 
Different colors indicate the availability of the nodes during training, calibration or testing. 
Solid and dashed lines represent edges used for training and edges within the test and calibration set. 
Predicting (1) corresponds to the transductive setting considered here.
(2) and (3) are examples of the inductive setting.
In road traffic forecasting, (1) may be the undetected traffic flow between two existing road junctions, e.g. for a (new) road where a traffic detector has not yet been installed.
(2) and (3) represent scenarios where new road junctions are constructed, connecting to existing ones or forming connections with each other to create new roads.}
\label{fig: transductive}
\end{figure}

The motivation for this approach can be observed in real-world traffic applications. We have an established area of the city monitored by traffic detectors, representing a fixed set of training edges. Simultaneously, in a new region, we randomly install traffic detectors on various roads. This process corresponds to the random division of the remaining edges into a calibration set and a test set.
In particular, we know the entire road systems, $A$, and the traffic volume of certain roads, $\Wtrain + \Wval$. 
The task is to predict the traffic volume of the remaining roads, $\Wtest$. 
During training, the model observes the nodes and leverages their features to make predictions. 
At inference time, the model deduces the edges that connect these nodes.
See Figure \ref{fig: transductive} for a graphical representation of our setup.

To predict the edge weights, we consider two GNN approaches. 
The first model is a link-prediction Graph Auto Encoder (GAE). 
Compared to the original GAE described in \citet{kipf2016variational}, we let the algorithm access the entire graph structure during training. 
The enhancement improves the model performance on edge weight prediction by allowing a better characterization of the edge environments. 
In the traffic forecasting setup, this means the road network remains unchanged at training and test time. 
Our second approach transforms the edge weight prediction problem into a node regression problem.
We convert the original graph into its line graph. 
The conversion preserves the graph structure, except for cases where the original graph is a triangle or a star network of four nodes \citep{whitney1992congruent}. 

The latter approach has a structural disadvantage. 
In the GAE method, $\Wtrain$ is used explicitly to update the node embeddings (see \eqref{eq: DiGAE} below).
In the line-graph approach, this is impossible because the training weights are used as labels.

\subsection{Graph Autoencoder} \label{subsec: GAE}
The GAE \citep{kipf2016variational} learns an embedding for the nodes of undirected unweighted graphs.
Using GAEs in link prediction tasks is a popular approach. 
The practice has been extended to isolated nodes \citep{ahn2021variational}, directed graphs \citep{kollias2022directed}, weighted graphs \citep{zulaika2022lwp}, and graphs with different edge types \citep{samanta2020nevae}. 

As in \citet{kipf2016variational}, we let $Z\in \mathbb{R}^{n\times d}$ be the node embedding matrix obtained from a base GNN model\footnote{In the following demonstration, we use the graph convolutional network (GCN) as the base GNN model.}, where $d$ represent the hidden dimension of node embeddings. 
The GNN model learns how to aggregate information from the neighbourhood of each node to update its features. 
The resulting embedding is 
\begin{equation}\label{eq: GAE}
\begin{split}
    &H^{(0)} = X, \\
    &H^{(l+1)} = \text{ReLU}\left( A H^{(l)} B^{(l)} \right), \, l=0, \cdots, L-1, \\
    &Z = A H^{(L)} B^{(L)},
\end{split}
\end{equation}
where $H^{(l)}$ and $B^{(l)}$ are the node feature  and weight matrices at layer $l$.
To ease the notation, we let 
\begin{equation}
Z = f_{\theta}(X, A),
\label{eq:gae encoder}
\end{equation}
where the structure of the encoder, $f_{\theta}$, is defined in \eqref{eq: GAE} and $\theta = \{B^{(l)}\}_{l=0, \cdots, L}$ is a learnable parameter.
We reconstruct the binary adjacency matrix from the inner product between node embeddings, i.e. 
\begin{equation}\label{eq:gae decoder}
P({\hat{A}\,|\,Z}) = \prod_{i=1}^n\prod_{j=1}^n P(\hat{A}_{ij}\,|\,Z_i,Z_j)\, , \,\,\, \text{with} \; P(\hat{A}_{ij}=1\,|\,Z_i,Z_j) = \sigma(Z_i^\top Z_j) \, ,
\end{equation}
where $\hat{A}$ is the reconstructed binary adjacency matrix and $\sigma(\cdot)$ is the logistic sigmoid function. 
A more flexible version of the above is the directed GAE of \citep{kollias2022directed}.  
For highlighting the roles of nodes as either a source or a target in directed graphs,
a source and a target embeddings, $Z^S$ and $Z^T$, replace the single node embedding of \eqref{eq: GAE}. 
The encoder structure becomes 
\begin{equation}\label{eq: DiGAE}
\begin{split}
    &H_S^{(0)} = X, \quad H_T^{(0)} = X,\\
    &H_S^{(l+1)} = \text{ReLU}\left( {\Wtrain} H_T^{(l)} B_T^{(l)} \right), \\ 
    &H_T^{(l+1)} = \text{ReLU}\left( {\Wtrain}^\top H_S^{(l)} B_S^{(l)} \right), \, l=0, \cdots, L-1, \\
    &{Z^S} = {\Wtrain} H_T^{(L)} B_T^{(L)}, \quad {Z^T} = {\Wtrain}^\top H_S^{(L)} B_S^{(L)}, \\
\end{split}
\end{equation}
where $H_S^{(l)} $ and $H_T^{(l)}$ and $B_S^{(l)}$ and $ B_T^{(l)}$ are the source and target feature and weight matrices at layer $l$.
Compared to (\ref{eq: GAE}), we also replace the binary adjacency matrix with the weighted adjacency matrix $\Wtrain$ (\ref{eq: weighted adj train}) which effectively leverages the entire graph structure.

The predicted weighted adjacency matrix is
\begin{equation}
    \hat{W} = Z^S {Z^T}^\top.
\end{equation}
To optimize the GNNs parameters, we minimize 
\begin{equation} \label{eq: train DiGAE}
    \loss_{\textrm{GAE}} = \| \Atrain \odot \hat{W} -\Wtrain \|_F.    
\end{equation}
through gradient descent.
We train the model until convergence and then select the parameters that minimize $\loss_{\textrm{GAE}}$ on the validation set, $\Wval$.

\subsection{Line Graph Neural Network} \label{subsec: line graph} 
An alternative approach to predict edge weights is through an edge-centric line graph model. 
The idea is to convert the weight prediction task into a node regression problem. 
We define a line-graph GNN and train it with standard message-passing techniques. 
Given a weighted directed graph, $G$, the corresponding line graph, $L(G)$, is a graph such that each node of $L(G)$ represents an edge of G. Two nodes of $L(G)$ are adjacent if and only if their corresponding edges share a common endpoint in $G$. 
Equivalently, $L(G)$ is the intersection graph of the edges of $G$.
Each edge of $G$ becomes a node of $L(G)$, labelled by the set of its two endpoints. 
Let $L = L(G)$ and $X^L$ be the node feature matrix of $L$,
To obtain $X^L$, we combine the node features of the corresponding source node and target node in the original graph. 
We then define a GNN to process the nodes and the binary adjacency matrix of $L$.
The predicted node value are
\begin{equation}\label{eq: LGNN}
    Z^L = f_{\theta}(X^L, A^L),
\end{equation}
Similar to the GAE approach, we tune the GNN parameters by minimizing  
\begin{equation}
    \loss_{\textrm{LGNN}} = \sum_{(i, j) \in \Etrain} \left(Z^L_{(i, j)} - \Wtrain_{ij} \right) ^ 2.
\end{equation}
The load prediction task becomes a node regression problem, with node values used as labels. 
We split the (augmented) node set of $L$ into training, test, and calibration sets.
The GAE training weights correspond to the values of the training nodes of $L$.

\section{Related Work}
\subsection{Link Prediction}
Link prediction refers to the task of forecasting node connections in a graph. 
Its practical uses include predicting future friendships in social networks \citep{liben2003link}, detecting forthcoming collaborations in academic coauthor networks \citep{adamic2003friends}, identifying protein-protein interactions in biological networks \citep{lei2013novel}, and suggesting items in recommendation systems \citep{yilmaz2023link}. 
Traditional methods depended on heuristic node-similarity scores or latent node embedding. 
GNNs usually outperform these methods because they learn from the graph structure and node or edge features \citep{zhang2018link}.  

Current GNN-based link prediction methods \citep{liben2003link, zhang2018link, kipf2016variational, berg2017graph} ignore the edges between training and testing nodes \citep{chen2018network}. 
We address this shortcoming by assigning an arbitrary weight, $\delta$ in \eqref{eq: weighted adj train} to the calibration and test edges. 
This makes the binary adjacency matrix of the entire graph available to the model at training time (see Proposition \ref{prop:exchangeability}).

We do not employ Variational Graph Autoencoders (VGAEs) \citep{kipf2016variational} because they assume the node embedding is Gaussian distributed. 
As we obtain the edge weights from the inner product of two node embeddings, the assumption would restrict the distribution of the model outputs \cite{mallik2011distribution} and hamper the nonparametric advantages offered by Conformal Quantile Regression (CQR).

\subsection{Traffic Prediction}
Many existing studies on traffic forecasting primarily focus on developing deterministic prediction models \citep{bui2022spatial}.
Traffic applications, however, often require uncertainty estimates for future scenarios. 
\citet{zhou2020variational} incorporate the uncertainty in the node representations. 
In \citet{xu2023air}, a Bayesian ensemble of GNN models combines posterior distributions of density forecasts for large-scale prediction. 
\citep{maas2020uncertainty} combine Quantile Regression (QR) and Graph WaveNet 
to estimate the quantiles of the load distribution.

Traditionally, traffic forecasting is approached as a node-level regression problem \citep{cui2019traffic, jiang2022graph}, i.e. nodes and edges in a graph represent monitoring stations and their connections.
We adopt an edge-centric approach, i.e. we predict traffic flow over road segments through edge regression. 
Interestingly, the strategy aligns with several real-world setups, e.g. the Smart City Blueprint for Hong Kong 2.0, which emphasizes monitoring road segments (edges) rather than intersections (nodes) \citep{office2019smart}.

\subsection{Conformal Prediction (CP)}
CP provides prediction regions for variables of interest \citep{vovk2005algorithmic}. 
Replacing a model's point predictions with prediction regions is equivalent to estimating the model uncertainty.
Recent applications of CP range from pandemic-driven passenger booking systems \citep{werner2021evaluation} to smartwatch-based detection of coughing and sneezing events \citep{nguyen2018cover}, or model calibration in the scikit-learn library \citep{sweidan2021probabilistic}.

Standard CP uncertainty estimation requires training and testing to be exchangeable \citep{vovk2005algorithmic}.
Relaxing the exchangeability assumption would make CP applicable to various real-world scenarios, e.g. covariate-shifted \citep{tibshirani2019conformal} data, and time-series forecast \citep{gibbs2021adaptive, nettasinghe2023extending}, and graph-based applications \citep{zargarbashi23conformal, luo2023anomalous}. \citet{tibshirani2019conformal} extends the existing framework to handle situations where the training and test covariate distributions are different. 
\citet{barber2023conformal} addresses the more challenging distribution drift case. 
\citet{huang2023uncertainty} applies similar ideas to a graph-based model. 
In \citep{clarkson2023distribution}, the ERC method (Section \ref{subsec: ERC}) is adapted to produce Neighbourhood Adaptive Prediction Sets (NAPS). 
The method assigns higher weights to calibration nodes closer to the test node.
This restores the exchangeability of the conformity scores associated with the calibration set.

\section{Conformalized Graph Autoencoder}\label{sec: conformal}
In this section, we describe how to integrate CP uncertainty estimation \citep{vovk2005algorithmic} into a GAE model (Section \ref{subsec: GAE}).

\subsection{Conformal Prediction}\label{subsec: CP}
We assume we have access to the graph structure, $A$, the node features, $X$, and the weighted adjacency matrix $\Wtrain$ (\ref{eq: weighted adj train}).
Let $(a, b)$ be the endpoints of a test edge.
We aim to generate a prediction interval, $ C_{ab} = \left(f_{\theta}((a, b), A, X, \Wtrain) \right) \subset \mathbb{R}$, for the weight of the such a test edge.
The prediction interval should be marginally valid, i.e. it should obey 
\begin{equation}\label{eq: desired coverage}
P\left( W_{ab} \in C_{ab} \right) \geq 1 - \alpha, 
\end{equation}
where $\alpha \in (0, 1)$ is a user-defined error rate. 
The probability is over the data-generating distribution. 
For efficiency, we focus on the split CP approach \citep{papadopoulos2002inductive}, using the training edge set $\Etrain$ for training and the calibration edge set $\Ecalib$ for calibration.
$\Etrain$ is used to fit the prediction model, $f_{\theta}$, and a \score{} is calculated for each sample in $\Ecalib$.
The \score{} evaluates how well the predictions match the observed labels.
Lower scores usually indicate better predictions.
Given a user-specified error rate, $\alpha$, and the endpoint of a test edge, $( a, b)$, we compute the corresponding prediction interval, $C_{ab}$, using the $(1-\alpha)$-th sample quantile of the calibration conformity scores. 
If the calibration edges and $(a, b)$ are exchangeable, $C_{ab}$ has the required coverage (\ref{eq: desired coverage}).
This implies that the exchangeability requirement is only necessary between the calibration and test edges, aligning with the methodology of \citet{huang2023uncertainty}. In real-world traffic applications, we often encounter a fixed set of training edges, for instance, a designated area in a city with well-documented traffic flow data. Furthermore, a separate set might serve as both calibration and test sites, where traffic detectors are placed randomly. This arrangement ensures that the calibration and test edges are exchangeable.

Algorithm \ref{alg: split CP} shows how to use split CP with a GAE model for predicting edge weights. 
Proposition \ref{prop:exchangeability} shows that the load prediction intervals generated by applying split CP to the GAE model are marginally valid in the sense of \eqref{eq: desired coverage}.

\begin{algorithm}
\caption{Split Conformal Prediction for Graph Autoencoder}
\label{alg: split CP}
\hspace*{\algorithmicindent} \textbf{Input:} The binary adjacency matrix $A \in \{0, 1\}^{n\times n}$, node features $X\in \mathbb{R}^{n\times f}$, training edges and their weights $\Etrain$, $\Wtrain$, calibration edges and their weights $\Ecalib$, $\Wcalib$, and test edges $\Etest$, user-specified error rate $\alpha \in (0,1)$, GAE model $f_\theta$ with trainable parameter $\theta$.\\
\begin{algorithmic}[1]
\State Train the model $f_\theta$ with $\Wtrain$ according to (\ref{eq: train DiGAE}).
\State Compute the \score{} which measures how well the calibration edge weights $\Wcalib$ agree with the model $f_\theta$:
\begin{equation}\label{eq: score}
    V_{ij} = \left| f_\theta \left((i, j); A, X, \Wtrain \right) - \Wcalib_{ij} \right|, \; (i, j) \in \Ecalib.
\end{equation}
\State Compute $d =$ the $k$th smallest value in $\{V_{ij}\}$, where $k=\lceil(|\Ecalib| +1)(1-\alpha)\rceil$.
\State Construct a prediction interval for test edges:  
\begin{equation}
\begin{split}
   & C_{ab} \\
= & \Big[f_\theta\left( (a, b); A, X, \Wtrain \right) - d, f_\theta\left( (a, b); A, X, \Wtrain \right) + d \Big], \; (a, b) \in \Etest. \nonumber
\end{split}
\end{equation}
\end{algorithmic}
\hspace*{\algorithmicindent} \textbf{Output:} Prediction intervals for the test edges $(a, b) \in \Etest$ with the coverage guarantee:
\begin{equation}
    P\big(\Wtest_{ab} \in C_{ab} \big) \geq 1 - \alpha.
\end{equation}
\end{algorithm}
The GAE model in Algorithm \ref{alg: split CP} uses the graph structure, i.e. the binary adjacency matrix, $A$, and the training edge weights, $\Wtrain$, and the node features, $X$.
As the order of the nodes is arbitrary, the \score{} the calibration and test samples are exchangeable (see Assumption 1 of \citet{huang2023uncertainty}). 
Intuitively, varying the choice of the calibration and test sets will not statistically alter the \score{}.

\subsection{Conformal Quantile Regression}\label{subsec: CQR}
The GAE model of Algorithm\ref{alg: split CP} is a mean regression model, i.e. its output is the conditional expectation of the object label given the object features. 
In this case, the model learns a feature embedding for each node and generates the prediction given the pair of nodes connected by an edge $(i, j)$. 
Since GAE predicts the edge weights based on the embeddings of the two adjacent nodes, its outputs are conditionally independent given the node embeddings \citep{jia2020residual}. 
The associated prediction intervals are marginally valid by construction, i.e. the estimated model uncertainty is constant over the entire graph.
This may make the prediction bands inefficient if the data are heteroscedastic \citep{lei2018distribution}
A possible way out is CQR, which combines the advantages of CP and QR when handling heteroscedastic data \citep{romano2019conformalized}.

We improve GAE's computational efficiency by making the encoder (Section \ref{subsec: GAE}) produce a triple output, i.e. three embeddings for each node. 
The decoder then aligns these embeddings to the mean, the $\alpha/2$ quantile, and the $(1-\alpha/2)$ quantile of the predicted edge weights. 
This differs from having three single-output GAE encoders because most network parameters are shared across the three embeddings.
Let $\hat{W}$, $\hat{W}^{\alpha/2}$, and $\hat{W}^{1 - \alpha/2}$ be the mean, $\alpha/2$, and $(1-\alpha/2)$ quantiles of the edge weights, i.e.,
\begin{equation}\label{eq: CQR output}
    f_\theta\left( (i, j); A, X, \Wtrain \right) = \left[ \hat{W}_{ij}, \hat{W}^{\alpha/2}_{ij}, \hat{W}^{1 - \alpha/2}_{ij}  \right]
\end{equation}
We train the embedding by minimizing 
\begin{equation}\label{eq: train CQR-GAE}
    \loss_{\textrm{CQR-GAE}} = \loss_{\textrm{GAE}} + \sum_{(i, j) \in \Etrain} \rho_{\alpha/2}(\Wtrain_{ij}, \hat{W}^{\alpha/2}_{ij}) + \rho_{1 - \alpha/2}(\Wtrain_{ij}, \hat{W}^{1-\alpha/2}_{ij}), 
\end{equation}
where $\loss_{\textrm{GAE}}$ is the squared error loss defined in (\ref{eq: train DiGAE})
The second term is the pinball loss of \citet{steinwart2011estimating, romano2019conformalized}, defined as 
\begin{equation}
    \rho_{\alpha}(y, \hat{y}) \coloneqq 
    \begin{cases}
        \alpha (y - \hat{y}) & \textrm{if } y > \hat{y} \\
        (1 - \alpha) (y - \hat{y}) & \textrm{otherwise}
    \end{cases}
\end{equation}
The first term is added to train the mean estimator, $\hat W$.
Algorithm \ref{alg: CQR} describes how to obtain the prediction intervals in this setup. Contrary to the CP conformity score (\ref{eq: score}), the CQR conformity score (\ref{eq: CQR score}) considers both undercoverage and overcoverage scenarios.

\begin{algorithm}
\caption{Conformal Quantile Regression for Graph Autoencoder}
\label{alg: CQR}
\hspace*{\algorithmicindent} \textbf{Input:} The binary adjacency matrix $A \in \{0, 1\}^{n\times n}$, node features $X\in \mathbb{R}^{n\times f}$, training edges and their weights $\Etrain$, $\Wtrain$, calibration edges and their weights $\Ecalib$, $\Wcalib$, and test edges $\Etest$, user-specified error rate $\alpha \in (0,1)$, GAE model $f_\theta$ with trainable parameter $\theta$.\\
\begin{algorithmic}[1]
\State Train the model $f_\theta$ with $\Wtrain$ according to (\ref{eq: train CQR-GAE}).
\State Compute the \score{} which quantifies the residual of the calibration edge weights $\Wcalib$ projected onto the nearest quantile produced by $f_\theta$ (\ref{eq: CQR output}):
\begin{equation}\label{eq: CQR score}
    V_{ij} = \max\left \{ \hat{W}^{\alpha/2}_{ij} - \Wcalib_{ij}, \Wcalib_{ij} - \hat{W}^{1-\alpha/2}_{ij}  \right\}, \; (i, j) \in \Ecalib.
\end{equation}
\State Compute $d =$ the $k$th smallest value in $\{V_{ij}\}$, where $k=\lceil(|\Ecalib| +1)(1-\alpha)\rceil$;
\State Construct a prediction interval for test edges:  
\begin{equation}
    C_{ab} = \Big[\hat{W}^{\alpha/2}_{ab} - d, \hat{W}^{1-\alpha/2}_{ab} + d \Big], \; (a, b) \in \Etest. \nonumber
\end{equation}
\end{algorithmic}
\hspace*{\algorithmicindent} \textbf{Output:} Prediction intervals for the test edges $(a, b) \in \Etest$ with the coverage guarantee:
\begin{equation}
    P\big(\Wtest_{ab} \in C_{ab} \big) \geq 1 - \alpha.
\end{equation}
\end{algorithm}

\subsection{Error Reweighted Conformal Approach}\label{subsec: ERC}
When calibration and test samples are exchangeable, both CP (Section \ref{subsec: CP}) and CQR (Section \ref{subsec: CQR}) yield prediction intervals that meet the marginal coverage condition (\ref{eq: desired coverage}).
Local adaptability can be improved by adding an Error-Reweighting (ER) factor as in \citet{papadopoulos2011regression, lei2018distribution}.
The idea is to assign covariate-dependent weights to the errors, thereby mitigating the impact of heteroscedasticity on the accuracy and reliability of the predictions. 

In CP, we use MC dropout \citep{gal2016dropout} to assess the variability of model output. 
MC dropout is employed during evaluation and generates multiple predictions. 
We use the standard deviation of these predictions as a proxy of the residual. And the \score{} of CP in (\ref{eq: score}) with 
\begin{equation}
    V^{\textrm{ERC}}_{ij} =  \frac{\left| f_\theta \left((i, j); A, X, \Wtrain \right) - \Wcalib_{ij} \right|}{s_{ij}^{\textrm{MC}} + \epsilon}, \; (i, j) \in \Ecalib,
\end{equation}
where $s_{ij}^{\textrm{MC}} = \sqrt{\frac{1}{K-1} \sum_{k=1}^{K}(f^{k}_\theta \left( (i, j); A, X, \Wtrain \right) - \Bar{f}_\theta \left( (i, j); A, X, \Wtrain \right))^2}$ is the standard deviation of the model evaluations using MC dropout, $\epsilon > 0$ is a regularization hyperparameter to be determined by cross-validation \citep{lei2018distribution}. 
We set the number of model evaluations as $K=1000$ in numerical experiments. 

The empirical simulations of \citet{guan2023localized} show that combining the CQR and ER approaches produces efficient and locally adaptive intervals.
Besides a prediction model, this method requires training a residual model, which captures the local variations present in the data. 
In CQR, the residual model comes at no extra cost, as we obtain it from the distance from the $\alpha/2$-th and $(1-\alpha/2)$-th predicted quantiles. 
More concretely, we replace the \score{} of CQR in (\ref{eq: CQR score}) with \begin{equation}
    V^{\textrm{ERC}}_{ij} = \max\left \{ \frac{\hat{W}^{\alpha/2}_{ij} - \Wcalib_{ij}}{\big|\hat{W}^{1-\alpha/2}_{ij} - \hat{W}^{\alpha/2}_{ij}\big|}, \frac{\Wcalib_{ij} - \hat{W}^{1-\alpha/2}_{ij}}{\big|\hat{W}^{1-\alpha/2}_{ij} - \hat{W}^{\alpha/2}_{ij}\big|}  \right\}, \; (i, j) \in \Ecalib,
\end{equation}

Let $d^{\textrm{ERC}} =$ be the $k$th smallest value in $\{V^{\textrm{ERC}}_{ij}\}$, where $k=\lceil(n/2 +1)(1-\alpha)\rceil$.
The ER prediction intervals are 
\begin{equation}\label{eq: interval ERC CP}
\begin{split}
   C_{ab} = \Big[ & f_\theta \left((a, b); A, X, \Wtrain \right) - d^{\textrm{ERC}} \big( s_{ab}^{\textrm{MC}} + \epsilon \big),  \\
   & f_\theta \left((a, b); A, X, \Wtrain \right) + d^{\textrm{ERC}} \big( s_{ab}^{\textrm{MC}} + \epsilon \big) \Big], \; (a, b) \in \Etest,
\end{split}
\end{equation}
for CP-ERC, and
\begin{equation}\label{eq: interval ERC}
\begin{split}
   C_{ab} =  \Big[ &\hat{W}^{\alpha/2}_{ab} - d^{\textrm{ERC}} \big|\hat{W}^{1-\alpha/2}_{ab} - \hat{W}^{\alpha/2}_{ab} \big|,  \\
   & \hat{W}^{1-\alpha/2}_{ab} + d^{\textrm{ERC}}\big|\hat{W}^{1-\alpha/2}_{ab} - \hat{W}^{\alpha/2}_{ab}\big| \Big], \; (a, b) \in \Etest,
\end{split}
\end{equation}
for CQR-ERC.

\begin{proposition}\label{prop:exchangeability}
The prediction intervals generated by split CP (Algorithm \ref{alg: split CP}), CQR (Algorithm \ref{alg: CQR}), and ERC (Section \ref{subsec: ERC}), are marginally valid, i.e. obey (\ref{eq: desired coverage}).
\end{proposition} 

\begin{proof}
    First, we show that the calibration and test conformity scores defined in (\ref{eq: score}) are exchangeable. 
    Given the entire graph structure, $A$, all the node features, $X$, and the edge weights of the training edges, $\Wtrain$, 
    the node embeddings are trained based on $\Wtrain$, and the edge weights in the remaining $\Ect$ are set randomly, the division of $\Ect$ into $\Ecalib$ and $\Etest$ have no impact on the training process. Consequently, the \score{s} for $\Ecalib$ and $\Etest$ are exchangeable.
    In practice, we split $\Ect$ into $\Ecalib$ and $\Etest$ randomly (as detailed in Section \ref{sec: empirical})  by converting the graph into its line graph and then selecting nodes uniformly at random.

    We also explore an alternative proof which is equivalent to the proof in \citet{huang2023uncertainty} but applied within a line graph setting. 
    Consider the original graph $G = (V, E)$ and its corresponding line graph $G' = (V', E')$, where $V' = E$ and $E'$ denotes adjacency between edges in $G$. After randomly dividing $E$ into $\Etrain$ and $\Ect$, and further splitting $\Ect$ into $\Ecalib$ and $\Etest$, the edges of $G$ transforms into nodes in $G'$. This setup mirrors the node division in the line graph. We train node embeddings on $\Etrain$ using a graph autoencoder, which aligns with fixing the training node set in $G'$. Given this fixed training set, any permutation and division of $\Ect$ (which corresponds to nodes in $G'$) doesn't affect the training, and thus the \score{s} computed for $\Ecalib$ and $\Etest$ are exchangeable. 
    
    Given this exchangeability of \score{s}, the validity of the prediction interval produced by CP and CQR follows from Theorem 2.2 of \citet{lei2018distribution} and Theorem 1 of \citet{romano2019conformalized}.    
    Let $V$ be the \score{} of CQR (\ref{eq: CQR score}).
    The ERC approach performs a monotone transformation of $V$, defined as 
    \begin{equation}
        \Phi_{ij}(V) = \frac{V}{\big|\hat{W}^{1-\alpha/2}_{ij} - \hat{W}^{\alpha/2}_{ij}\big|}, 
    \end{equation}
    where $i$ and $j$ are two nodes in the graph\footnote{The nodes are represented by the  node features, $X_i, X_j$, and the embeddings,  $Z_i, Z_j$.}. 
    For all $(i, j)$ and all $V$,  $\Phi_{ij}^{\prime}(V)  = \frac{\partial\Phi_{ij}(V)}{\partial V }> 0 $, i.e. the transformation is strictly monotonic in $V$. 
    This implies $\Phi_{ab}$ is invertible for any test edge, $(a, b)$.
    Let $\Phi_{ab}^{-1}$ be the inverse of $\Phi_{ab}$. 
    The Inverse Function Theorem implies $\Phi_{ab}^{{-1}}$ is also strictly increasing.  
    Now suppose $d^{\textrm{ERC}}$ is the $k$th smallest value in $\{V^{\textrm{ERC}}_{ij}\} = \{\Phi_{ij}(V_{ij}) \}$, $k=\lceil(|\Ecalib| +1)(1-\alpha)\rceil$. Then for a test edge $(a, b)$, 
    \begin{equation}
        P(\Phi_{ab}(V_{ab}) \leq d^{\textrm{ERC}}) = \frac{\lceil(|\Ecalib| +1)(1-\alpha)\rceil}{|\Ecalib|+1}\geq 1 - \alpha.
    \end{equation} 
    Using the monotonicity of $\Phi^{-1}_{ab}$,
    \begin{equation}
    \begin{split}
        1 - \alpha & \leq P(\Phi_{ab}(V_{ab}) \leq V^{\textrm{ERC}}_k) \\
        &= P\left(V_{ab} \leq \Phi^{-1}_{ab}(d^{\textrm{ERC}}) \right) \\
        & = P\left(W_{ab} \in C_{ab} \right)
    \end{split}
    \end{equation}    
    The final equation is derived from the construction of the prediction interval (\ref{eq: interval ERC}) and the validity of CQR. 
    This shows that the prediction intervals based on the reweighted conformity scores are valid.
\end{proof}

\begin{figure}
\centerline{\includegraphics[width=\textwidth,clip=]{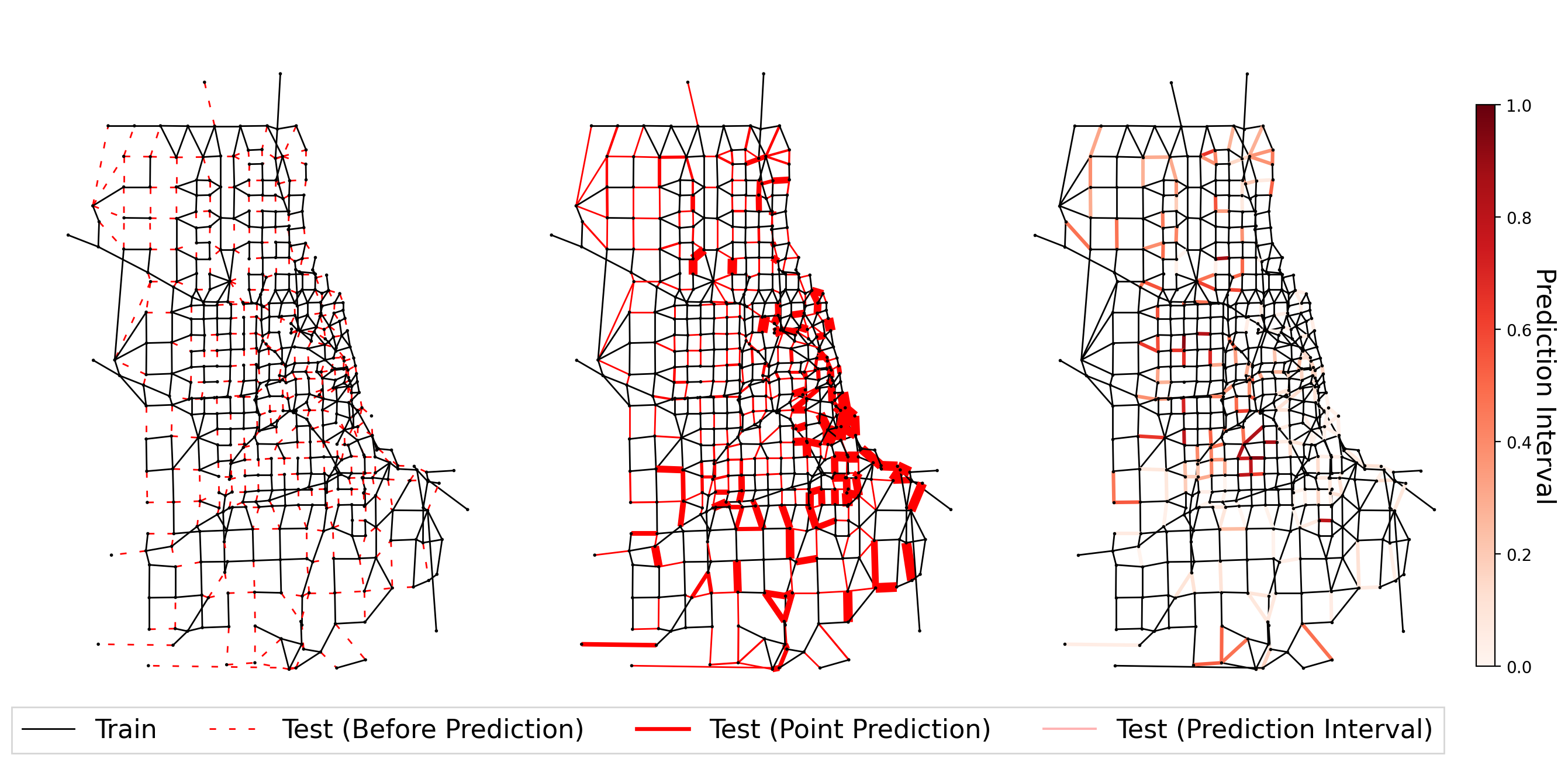}}
\small
\caption{The figure demonstrates the application of our proposed prediction models, which provide a coverage guarantee, using a snapshot of road network and traffic flow data from Chicago, IL, United States \citep{bar2021transportation}. The road network is divided into training roads (represented by black solid lines) and test roads (represented by red dashed lines). Our CQR-GAE model (Algorithm \ref{alg: CQR}) is developed to generate a prediction interval with a user-specified error rate of $\alpha=0.05$. The middle plot displays the predicted edge weights $\hat{W}$, where the line thickness increases proportionally with the predicted edge weights. The right plot illustrates the lengths of the prediction intervals, with darker lines indicating wider intervals or higher inefficiency (\ref{eq: ineff}).}
\label{fig: Chicago}
\end{figure}

\subsection{Comparison with Other Methods}

The NAPS method \citep{clarkson2023distribution} emphasizes inductive learning on graphs, which inherently assumes homophily since nodes closer to the target node are assigned more weight in constructing prediction sets for node prediction tasks. This stands in contrast to our approach, which focuses on transductive learning for edge prediction tasks on graphs. Additionally, the assumption of homophily may not hold in traffic networks \citep{xiao2023spatial}, as traffic conditions can vary; for example, a small road adjacent to a busy street might experience less traffic.
Diffusion Adaptive Prediction Sets (DAPS) \citep{zargarbashi23conformal} are applicable to the transductive learning setting but also requires homophily.

The primary innovation of our method lies in combining conformal prediction with a graph autoencoder framework to solve edge prediction problems. This is distinct from that of \citet{huang2023uncertainty}, who focuses on node prediction problems. 
Moreover, our experiments with the line graph demonstrate that the setup of an autoencoder framework and the transformation of an original graph into its line graph are not equivalent. This highlights the superiority of the autoencoder framework for addressing edge prediction problems, particularly through its use of the graph structure, which is specifically related to traffic-load applications where the graph structure remains constant but the edge weights vary. 

The Edge Exchangeable Model (EEM) method \citep{luo2023anomalous} applies an exchangeable distribution to the edges, positioning it within the inductive framework as it can handle unseen nodes during inference. However, it may fit poorly if the graph deviates from the edge exchangeability condition.

To the best of our knowledge, our research represents the first application of the ERC approach to graph-based prediction problems. In CQR-ERC, we highlight the benefits of incorporating localized variability into the construction of prediction intervals. The outputs $\hat{W}{ij}^{\alpha/2}$ and $\hat{W}{ij}^{1-\alpha/2}$ from the decoder, derived from node embeddings at various levels, naturally consider the neighborhood structure and attributes of adjacent nodes. Notably, CQR-ERC does not require additional training, contrasting with \citet{huang2023uncertainty}'s approach where the model undergoes fine-tuning using CP-aware objectives that require smooth approximations and specific training datasets. Consequently, CQR-ERC and their approach are orthogonal, allowing one to be implemented in conjunction with the other.

\section{Empirical Analysis}\label{sec: empirical}
In this section, we showcase the application of the proposed CP algorithms, Algorithm \ref{alg: split CP} and Algorithm \ref{alg: CQR}, to both GAE and LGNN models. We conduct a comparative analysis of the performance of these four models. The results demonstrate that CQR-GAE exhibits the highest level of efficiency among them. Additionally, the CQR-based models demonstrate enhanced adaptability to the data and have the capability to generate prediction intervals of varying lengths.

\vspace{0.1in}
\noindent
{\bf Dataset:}

We apply our proposed algorithm to a real-world traffic network, specifically the road network and traffic flow data from Chicago and Anaheim \citep{bar2021transportation}. The Chicago dataset consists of 541 nodes representing road junctions and 2150 edges representing road segments with directions, and the Anaheim dataset consists of 413 nodes and 858 edges. In this context, each node is characterized by a two-dimensional feature $X_i\in \mathbb{R}^{2}$ representing its coordinates, while each edge is associated with a weight that signifies the traffic volume passing through the corresponding road segment.

We adopt a similar procedure from \cite{jia2020residual, huang2023uncertainty}, and allocate 50\%, 10\%, and 40\% for the training set $\Etrain$, validation set $\Eval$, and the combined calibration and test set $\Ect$, respectively. Figure \ref{fig: Chicago} provides an example of how the Chicago network data is divided into training/val/test/calibration edges. Additionally, the prediction outcome of our proposed CQR-GAE (Algorithm \ref{alg: CQR}) is depicted. The median plot shows the predicted edge weights, while the right-hand plot shows the width of the prediction interval.

\begin{figure}
\centerline{\includegraphics[width=\textwidth,clip=]{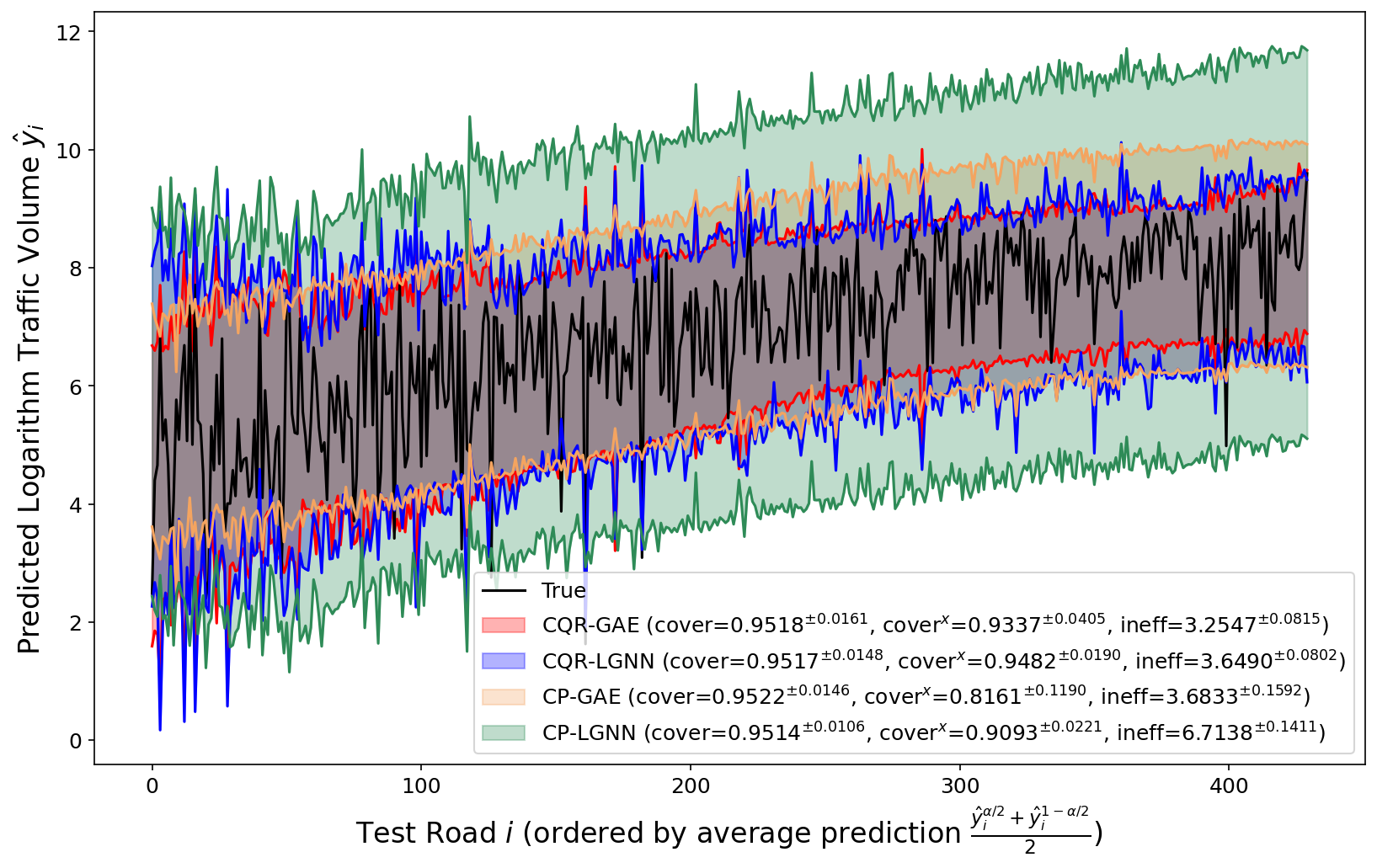}}
\small
\caption{The prediction interval generated by both CP (Algorithm \ref{alg: split CP}) and CQR (Algorithm \ref{alg: CQR}), which employ either GAE or LGNN as the GNN model, is constructed with a user-specified error rate of $\alpha=0.05$. All models successfully meet the coverage condition as their coverage (\ref{eq: cover}) surpasses $1-\alpha$. Notably, both CQR-GAE and CQR-LGNN outperform their CP counterparts in terms of inefficiency (\ref{eq: ineff}) and conditional coverage (\ref{eq: cond cover}). Among these, CQR-GAE achieves the best inefficiency.}
\label{fig: comparison}
\end{figure}

\vspace{0.1in}
\noindent
{\bf Evaluation Metrics:}

For evaluation, we use the marginal coverage, defined as 
\begin{equation}\label{eq: cover}
    \textrm{cover} = \frac{1}{|\Etest|} \sum_{(i,j)\in \Etest} \mathbbm{1}\big(\Wtest_{ij} \in C_{ij}\big),
\end{equation}
where $C_{ij}$ is prediction interval for edge $(i, j)$. Inefficiency is defined as
\begin{equation}\label{eq: ineff}
    \textrm{ineff} = \frac{1}{|\Etest|} \sum_{(i,j)\in \Etest} |C_{ij}|,
\end{equation}
which measures the average length of the prediction interval.

In addition to the marginal coverage, we also consider the conditional coverage. Specifically, we use the method of \citet{romano2020classification, cauchois2020knowing} to measure the coverage over a slab of the feature space $S_{v, a, b}= \big\{[X_i \mathbin\Vert X_j]\in \mathbb{R}^{2f}: a \leq v^\top x \leq b \big\}$, where $[X_i \mathbin\Vert X_j]$ denotes the node feature of two connected nodes of an edge $(i, j)$ and $v \in \mathbb{R}^{2f}$ and $a < b \in \mathbb{R}$ are chosen adversarially and independently from the data. For any prediction interval $f_\theta^*$ and $\delta \in (0, 1)$, the \textit{worst slice coverage} is defined as 
\begin{equation}\label{eq: cond cover}
\begin{split}
    \textrm{WSC}(f_\theta^*, \delta) &= \inf\limits_{\substack{v \in \mathbb{R}^{2f}, \\ a < b \in \mathbb{R}}} \big \{ P \big( \Wtest_{ij} \in C_{ij} \mid [X_i \mathbin\Vert X_j] \in S_{v, a, b}  \big) \\
    & \phantom{----------} \textrm{s.t. } P([X_i \mathbin\Vert X_j] \in S_{v, a, b}) \geq \delta  \big \}.
\end{split}
\end{equation}
We generate 1000 independent vectors $v$ on the unit sphere in $\mathbb{R}^{2f}$ and fine-tune the parameters $a, b, \delta$ using a grid search. Additionally, we utilize 25\% of the test data to estimate the optimal values for $v, a, b, \delta$, and calculate the conditional coverage with the leftover 75\%.

\vspace{0.1in}
\noindent
{\bf Models and baselines: }

We name the model that combines CP (Algorithm \ref{alg: split CP}) with GAE\footnote{We also experiment using GAE's directed variant, DiGAE. The corresponding models are CP-DiGAE, CQR-DiGAE, and CQR-ERC-DiGAE.} (Section \ref{subsec: GAE}) or LGNN (Section \ref{subsec: line graph}) as CP-GAE and CP-LGNN, respectively. Similarly, we name the models that use CQR (Algorithm \ref{alg: CQR}) as CQR-GAE and CQR-LGNN. We name the models that use ERC (Section \ref{subsec: ERC}) as CQR-ERC-GAE and CQR-ERC-LGNN.

To assess the coverage performance (\ref{eq: cover}), we employ quantile regression (QR) for GAE and LGNN as the baseline model. This model generates a prediction interval by optimizing the same loss function (\ref{eq: train CQR-GAE}), without calibration with CP. 

Considering that lower coverage results in higher efficiency, we limit our comparison to CP-based models and CQR-based models that achieve the coverage condition. We use four popular GNN models: GCN \citep{kipf2016semi}, GraphConv \citep{morris2019weisfeiler}, GAT \citep{velivckovic2017graph}, and GraphSAGE \citep{hamilton2017inductive}, as the base graph convolution layers for both CP and CQR based models.

\begin{table}
\label{tab:eff_all_models}
    \centering
    \begin{adjustbox}{width=\textwidth}
\begin{tabular}{|l|c|c|c|c|c|c|}
\toprule
GNN Model & \multicolumn{2}{c|}{GAE} & \multicolumn{2}{c|}{DiGAE} & \multicolumn{2}{c|}{LGNN}  \\ \cmidrule{1-7}
Anaheim  & cover & ineff & cover & ineff & cover & ineff\\\midrule
CP&$0.9545\std{0.0222}$
&$5.5550\std{0.6546}$
&$0.9546\std{0.0222}$
&$5.7775\std{0.5696}$
&$0.9526\std{0.0154}$
&$7.2667\std{0.3627}$\\
CP-ERC&$0.9543\std{0.0218}$
&$5.7091\std{0.7435}$
&$0.9546\std{0.0219}$
&$6.0347\std{0.6511}$
&$0.9529\std{0.0141}$
&$7.9131\std{0.6246}$\\
CQR &$0.9542\std{0.0221}$ 
&$5.2489\std{0.3378}$
&$0.9539\std{0.0219}$
&$5.2169\std{0.6917}$
&$0.9513\std{0.0232}$
&$5.4381\std{0.1453}$\\
CQR-ERC&${0.9548}\std{0.0222}$ 
&$5.2923\std{0.3648}$
&$0.9541\std{0.0220}$ 
&$\textbf{5.2082}\std{0.3681}$
&$0.9513\std{0.0231}$
&$5.4394\std{0.1466}$\\
QR&$\underline{0.9364}\std{0.0162}$
&${4.9680}\std{0.3067}$
&$\underline{0.9278}\std{0.0186}$ 
&${4.7075}\std{0.1529}$
&$\underline{0.9394}\std{0.0177}$
&${5.3239}\std{0.0751}$\\ \midrule
Chicago & cover & ineff & cover & ineff & cover & ineff\\\midrule
CP &$0.9512\std{0.0144}$
&$3.5787\std{0.1713}$
&$0.9515\std{0.0145}$
&$3.5727\std{0.2095}$
&$0.9500\std{0.0104}$
&$6.9399\std{0.1610}$ \\ 
CP-ERC &$0.9517\std{0.0142}$
&$3.5112\std{0.2216}$
&$0.9514\std{0.0147}$
&$3.5377\std{0.2247}$
&$0.9506\std{0.0094}$
&$7.4016\std{0.2926}$ \\ 
CQR  &$0.9515\std{0.0142} $ 
&${3.3987}\std{0.1088}$
&${0.9516}\std{0.0144}$
&$\mathbf{3.3578}\std{0.1382}$
&$0.9506\std{0.0165}$
&$3.4362\std{0.1029}$\\
CQR-ERC&$0.9507\std{0.0146}$ 
&$4.3706\std{0.4224}$
&$0.9504\std{0.0149}$ 
&$4.2859\std{0.5245}$
&$0.9507\std{0.0086}$
&$8.7069\std{0.7519}$\\
QR&$\underline{0.9497}\std{0.0110}$
&$3.3894\std{0.0871}$
&$0.9526\std{0.0088}$ 
&$3.3542\std{0.0992}$
&$0.9697\std{0.0083}$
&$3.6679\std{0.1094}$\\
\bottomrule
\end{tabular}    
    \end{adjustbox}
    \vspace{0.05in}
    \caption{Performance comparison of the proposed models in terms of coverage (\ref{eq: cover}) and inefficiency (\ref{eq: ineff}). GraphConv \citep{morris2019weisfeiler} is used as the graph convolutional layer for the GNN models. The best inefficiency for each dataset (excluding QR which does not meet the coverage condition) is highlighted in bold: CQR-ERC-DiGAE achieves the best inefficiency on the Anaheim dataset while CQR-DiGAE achieves the best inefficiency on the Chicago dataset. The baseline model, QR, does not meet the coverage condition and the results are underlined. CQR (including CQR-ERC) based models outperform CP based model in inefficiency. The performance is evaluated on 2 transportation network datasets.}
\end{table}

\vspace{0.1in}
\noindent
{\bf Result\footnote{The code is available at \url{https://github.com/luo-lorry/conformal-load-forecasting}.}:} 

For each dataset and model, we run the experiment 10 times and split the data into training, validation and the combined calibration and test sets. We conduct 100 random splits of calibration and testing edges to perform Algorithm \ref{alg: split CP} and Algorithm \ref{alg: CQR} and evaluate the empirical coverage.  

Table 1 indicates that QR does not meet the marginal coverage condition (\ref{eq: cover}). On the other hand, CP and CQR based models successfully meet the coverage condition, as indicated by their coverage (\ref{eq: cover}) surpassing $1-\alpha$. Table 1 also shows that GAE and DiGAE outperform LGNN, highlighting the efficacy of the autoencoder approach in weight prediction. 
Figure \ref{fig: comparison} illustrates the prediction interval produced by CP and CQR based models. These prediction intervals are constructed with a user-specified error rate of $\alpha=0.05$.

\begin{table}
\label{tab:eff_all_models1}
\centering
\begin{adjustbox}{width=\textwidth}
\begin{tabular}{|l|c|c|c|c|c|c|c|c|}
\toprule
GNN Layer & \multicolumn{2}{c|}{GraphConv} & \multicolumn{2}{c|}{SAGEConv}  & \multicolumn{2}{c|}{GCNConv} & \multicolumn{2}{c|}{GATConv} \\ \cmidrule{1-9}
Anaheim  & cover$^x$  & ineff & cover$^x$ & ineff & cover$^x$ & ineff& cover$^x$ & ineff\\\midrule
CP-GAE&$0.9293\std{0.0624} $ 
&$5.5550\std{0.6546}$
&$0.9287\std{0.0672}$
&$6.1299\std{0.6173}$
&$0.9366\std{0.0534}$ 
&$6.2583\std{0.6578}$
&$0.9341\std{0.0643}$
&$6.5190\std{0.6832}$\\
CP-DiGAE&$0.9326\std{0.0658}$ 
&$5.7775\std{0.5696}$
&$0.9248\std{0.0612}$
&$6.1795\std{0.6290}$
&$0.9237\std{0.0722}$ 
&$6.7250\std{1.0370}$
&$0.9144\std{0.0896}$
&$6.0298\std{0.5550}$\\
CP-LGNN&$0.9454\std{0.0305}$ 
&$7.2667\std{0.3627}$
&$0.9373\std{0.0373}$
&$6.6992\std{0.3739}$
&$0.9326\std{0.0405}$ 
&$6.6546\std{0.3979}$
&$0.9413\std{0.0344}$
&$6.8745\std{0.3756}$\\ \midrule
CP-ERC-GAE&$0.9218\std{0.0642}$ 
&$5.7091\std{0.7435}$
&$0.8944\std{0.0954}$
&$7.7478\std{1.1546}$
&$0.9267\std{0.0636}$ 
&$6.7686\std{0.9815}$
&$0.9233\std{0.0605}$
&$7.0980\std{0.6851}$\\
CP-ERC-DiGAE&$0.9205\std{0.0585}$ 
&$6.0347\std{0.6511}$
&$0.9140\std{0.0645}$
&$7.8277\std{1.0907}$
&$0.9227\std{0.0668}$ 
&$7.8166\std{2.1036}$
&$0.9195\std{0.0595}$
&$6.9144\std{0.8170}$\\
CP-ERC-LGNN&$0.9295\std{0.0343}$ 
&$7.9131\std{0.6246}$
&$0.9190\std{0.0423}$
&$8.3152\std{0.6159}$
&$0.9159\std{0.0351}$ 
&$9.5728\std{0.6349}$
&$0.9256\std{0.0372}$
&$7.8193\std{0.3708}$\\ \midrule
CQR-GAE&$\textbf{0.9548}\std{0.0219} $ 
&${5.2680}\std{0.3499}$
&$0.9535\std{0.0220}$
&$5.8272\std{0.2352}$
&$\textbf{0.9537}\std{0.0225}$
&$5.8324\std{0.1992}$
&$\textbf{0.9535}\std{0.0223}$
&$5.8893\std{0.2399}$\\
CQR-DiGAE&$0.8984\std{0.0926} $ 
&$\textbf{5.0580}\std{0.2792}$
&$0.8975\std{0.0982}$
&$\textbf{5.6483}\std{0.2399}$
&$0.9040\std{0.0873} $ 
&$5.7600\std{0.2960}$
&$0.9115\std{0.0691}$  
&$5.7889\std{0.2722}$\\
CQR-LGNN&$0.9010\std{0.0555}$ 
&${5.4381}\std{0.1453}$
&$0.9167\std{0.0480}$
&$5.9004\std{0.2302}$
&$0.9333\std{0.0430}$ 
&$6.1160\std{0.1818}$
&$0.9080\std{0.0607}$
&$6.0694\std{0.1861}$\\ \midrule  
CQR-ERC-GAE&$\textbf{0.9548}\std{0.0222} $ 
&${5.2923}\std{0.3648}$
&$\textbf{0.9537}\std{0.0221}$
&$5.8280\std{0.2340}$
&$\textbf{0.9538}\std{0.0225} $ 
&$5.8305\std{0.1966}$
&$\textbf{0.9535}\std{0.0223}$
&$5.8893\std{0.2421}$\\
CQR-ERC-DiGAE&$0.8953\std{0.0974}$ 
&${5.0600}\std{0.2837}$
&$0.8972\std{0.0983}$
&$5.6497\std{0.2479}$
&$0.9034\std{0.0867}$ 
&$\textbf{5.7595}\std{0.3153}$
&$0.9128\std{0.0690}$
&$\textbf{5.7840}\std{0.2725}$\\
CQR-ERC-LGNN&$0.9003\std{0.0556}$
&${5.4394}\std{0.1466}$
&$0.9171\std{0.0476}$ 
&$5.9012\std{0.2284}$
&$0.9326\std{0.0442}$
&$6.1131\std{0.1791}$
&$0.9326\std{0.0442}$
&$6.1131\std{0.1791}$\\ \midrule
Chicago  & cover$^x$  & ineff & cover$^x$ & ineff & cover$^x$ & ineff& cover$^x$ & ineff\\\midrule
CP-GAE&$0.8557\std{0.1165} $ 
&$3.5787\std{0.1713}$
&$0.8463\std{0.1302}$
&$3.6475\std{0.2073}$
&$0.8499\std{0.1350}$ 
&$3.6098\std{0.2233}$
&$0.8534\std{0.1185}$
&$3.8353\std{0.1860}$\\
CP-DiGAE&$0.8526\std{0.1276} $ 
&$3.5727\std{0.2095}$
&$0.8616\std{0.1202}$
&$3.5754\std{0.2280}$
&$0.8179\std{0.1310}$ 
&$3.5800\std{0.2136}$
&$0.8541\std{0.1109}$
&$3.7372\std{0.2165}$\\
CP-LGNN&$0.9049\std{0.0331}$ 
&$6.9399\std{0.1610}$
&$0.9049\std{0.0344}$
&$6.8360\std{0.1665}$
&$0.9067\std{0.0315}$ 
&$6.7206\std{0.1618}$
&$0.9052\std{0.0302}$
&$6.7616\std{0.1621}$\\ \midrule
CP-ERC-GAE&$0.9140\std{0.0606} $ 
&$3.5112\std{0.2216}$
&$0.8925\std{0.0778}$
&$3.7330\std{0.2980}$
&$0.8885\std{0.0948}$ 
&$5.0146\std{1.5911}$
&$0.8572\std{0.0855}$
&$4.5945\std{0.5359}$\\
CP-ERC-DiGAE&$0.9093\std{0.0682}$ 
&$3.5377\std{0.2247}$
&$0.8952\std{0.0724}$
&$3.6713\std{0.2693}$
&$0.8847\std{0.1028}$ 
&$4.8554\std{1.6399}$
&$0.8832\std{0.0857}$
&$4.0125\std{0.3946}$\\
CP-ERC-LGNN&$0.9293\std{0.0238}$ 
&$7.4016\std{0.2926}$
&$0.9273\std{0.0245}$
&$7.1055\std{0.2248}$
&$0.9245\std{0.0305}$ 
&$11.1934\std{1.7155}$
&$0.9291\std{0.0193}$
&$7.1543\std{0.2534}$\\ \midrule
CQR-GAE&$\textbf{0.9514}\std{0.0144} $ 
&${3.3652}\std{0.1312}$
&$\textbf{0.9517}\std{0.0141}$
&$3.5878\std{0.2107}$
&$\textbf{0.9516}\std{0.0144}$
&$3.5131\std{0.1353}$
&$0.9515\std{0.0143}$
&$\textbf{3.5428}\std{0.2027}$\\
CQR-DiGAE&$0.9205\std{0.0498} $ 
&$\textbf{3.3135}\std{0.1172}$
&$0.9223\std{0.0469}$
&$\textbf{3.3872}\std{0.1260}$
&$0.9250\std{0.0479} $ 
&$\textbf{3.4241}\std{0.1271}$
&$0.9089\std{0.0611}$
&$3.6158\std{0.2348}$\\
CQR-LGNN&$0.9284\std{0.0296}$ 
&${3.4362}\std{0.1029}$
&$0.9305\std{0.0258}$
&$3.4844\std{0.1233}$
&$0.9290\std{0.0284}$ 
&$3.6514\std{0.1050}$
&$0.9379\std{0.0261}$
&$4.0805\std{0.5445}$\\ \midrule  
CQR-ERC-GAE&$0.9512\std{0.0145} $ 
&${3.3648}\std{0.1358}$
&$\textbf{0.9517}\std{0.0142}$
&$3.5840\std{0.2150}$
&$0.9514\std{0.0146} $ 
&$3.5126\std{0.1367}$
&$\textbf{0.9516}\std{0.0142}$
&$3.5467\std{0.2019}$\\
CQR-ERC-DiGAE&$0.9199\std{0.0522} $ 
&${3.3145}\std{0.1184}$
&$0.9216\std{0.0486}$
&$\textbf{3.3873}\std{0.1272} $
&$0.9261\std{0.0486}$
&$3.4257\std{0.1251}$
&$0.9092\std{0.0620}$ 
&$3.6147\std{0.2328}$\\
CQR-ERC-LGNN&$0.9279\std{0.0304}$
&${3.4325}\std{0.1004}$
&$0.9285\std{0.0277}$ 
&$3.4658\std{0.1192}$
&$0.9273\std{0.0308}$
&$3.6185\std{0.1023}$
&$0.9376\std{0.0269}$
&$4.0203\std{0.5530}$\\ \bottomrule
\end{tabular}    
    \end{adjustbox}
    \vspace{0.05in}\caption{Performance comparison of the proposed models, based on the conditional coverage (\ref{eq: cond cover}) and inefficiency (\ref{eq: ineff}). The models were tested using several widely-used graph convolutional layers, including GraphConv \citep{morris2019weisfeiler}, SAGEConv \citep{hamilton2017inductive}, GCNConv \citep{kipf2016semi}, and GATConv \citep{velivckovic2017graph}. The best conditional coverage and inefficiency for each graph convolutional layer is highlighted in bold. Across diverse datasets and graph convolutional layers, CQR-GAE and CAR-ERC-GAE demonstrate strong performance in both inefficiency and conditional coverage, while CQR-DiGAE and CQR-ERC-DiGAE excels in minimizing inefficiency.}
\end{table}

Furthermore, Table 2 shows that CQR based models outperform their CP counterparts in terms of inefficiency (\ref{eq: ineff}) and conditional coverage (\ref{eq: cond cover}). This indicates that the CQR variants produce a better balance between capturing the uncertainty in the predictions and maintaining a high level of coverage. Figure \ref{fig: comparison} additionally illustrates the CQR models' adaptability to the data by generating prediction intervals of varying sizes.

Furthermore, the discrepancy in results between the GAE/DiGAE model and the line graph model suggests that the direct application of \citet{huang2023uncertainty} for traffic prediction through node value prediction in the transformed line graph might not be as effective as our GAE/DiGAE approach. Our method involves training multi-level node embeddings and extracting multiple quantiles of edge values through the decoding of these node embeddings.

Additionally, CQR-ERC and its CQR counterpart exhibit comparable performance in terms of both conditional coverage and inefficiency. CP-ERC demonstrates improved conditional coverage compared to its CP counterpart, albeit at the cost of increased inefficiency when analyzing the Chicago network dataset. Specifically, when utilizing the GraphConv graph convolutional layer, CP-ERC exhibits superior efficiency compared to its CP counterpart. However, the scenario differs when analyzing the Anaheim network dataset. In this case, CP-ERC performs worse in terms of both conditional coverage and inefficiency. For ERC, tuning the regularization hyperparameter can be a notably challenging task. The performance of the ERC is largely sensitive to the choice of this hyperparameter, and ERC is likely to produce large prediction intervals \citep{sesia2020comparison}.

We also conduct an ablation study to assess the impact of setting the edge weights for the validation, calibration and test edge sets. Initially, we set these edge weights to zero, creating a scenario comparable to the line graph setting. Subsequently, we assign them the average edge weight from the training edges. Additionally, we assign weights randomly by bootstrapping the training edge weights and allocating the sampled values to them.

\begin{table}[h!]
    \centering
    \begin{adjustbox}{width=\textwidth}
    \begin{tabular}{lcccccc}
        \toprule
           & \multicolumn{2}{c}{Zero}   & \multicolumn{2}{c}{Mean}   & \multicolumn{2}{c}{Random} \\ 
           \midrule
          & cover   & ineff   & cov   & ineff   & cover   & ineff \\
        \midrule
        DiGAE on Anaheim   & $0.9063 \std{ 0.0713 } $ & $5.1402 \std{ 0.3131 } $ & $0.9094 \std{ 0.1182 } $ & $4.9396 \std{ 0.3067 } $ & $0.9106 \std{ 0.0896 } $ & $5.0127 \std{ 0.3143}$ \\
        GAE on Anaheim   & $0.8908 \std{ 0.0627 } $ & $5.2968 \std{ 0.3054 } $ & $0.8909 \std{ 0.0980 } $ & $5.1462 \std{ 0.2833 } $ & $0.8941 \std{ 0.0601 } $ & $5.1863 \std{ 0.2914}$ \\
        \midrule
        DiGAE on Chicago   & $0.9219 \std{ 0.0474 } $ & $3.4093 \std{ 0.1410 } $ & $0.9253 \std{ 0.0425 } $ & $3.1729 \std{ 0.1278 } $ & $0.9273 \std{ 0.0455 } $ & $3.1938 \std{ 0.1035}$ \\
        GAE on Chicago   & $0.9001 \std{ 0.0513 } $ & $3.4034 \std{ 0.1034 } $ & $0.9017 \std{ 0.0543 } $ & $3.2502 \std{ 0.0987 } $ & $0.9048 \std{ 0.0393 } $ & $3.2912 \std{ 0.0870}$ \\
        \bottomrule
    \end{tabular}
    \end{adjustbox}
    \caption{Results of the ablation study showing different weight assignments to non-training edges in the GAE/DiGAE models. The findings across both datasets for GAE and DiGAE demonstrate that these models effectively utilize graph structure information by assigning nonzero weights to edges outside the training set. This capability contributes to their enhanced performance compared to the LGNN method.}
    \label{tab:data_table}
\end{table}

As a potential future direction, we plan to conduct an analysis of conditional coverage for network-based features \citep{huang2023uncertainty}, such as clustering coefficients, betweenness centrality, PageRank, and others. By examining the impact of these network features on the performance of CP-ERC and its CP counterpart, we aim to gain further insights into the effectiveness of CP-ERC in capturing the conditional coverage of the network.

\section{Conclusion}\label{sec: conclusion}
In this paper, we proposed a graph neural network approach for the prediction of edge weights with guaranteed coverage. We use conformal prediction to calibrate GNN outputs and establish a prediction interval. To effectively handle heteroscedastic node features, we utilize conformal quantile regression and error reweighted conformal approaches. %

We conduct a comprehensive empirical evaluation on real-world transportation datasets to assess the performance of our proposed method. The results clearly demonstrate the superiority of our approach over baseline techniques in terms of both coverage and efficiency. Future work could focus on enhancing the efficiency of our method or extending its applicability to other types of networks. Instead of edge representation by decoding of node embeddings, alternative approaches such as edge embedding methods could be explored.

\end{document}